\RequirePackage{amsmath}
\documentclass{llncs}

\usepackage[utf8]{inputenc}
\usepackage[T1]{fontenc}
\usepackage{latexsym}

\usepackage{amssymb}

\usepackage{url}
\usepackage{cite}
\usepackage{hyperref}

\usepackage{graphicx}
\usepackage{tabularx}
\usepackage{array}
\usepackage{mdwmath}
\usepackage{mdwtab}
\usepackage{float}

\usepackage[cmyk]{xcolor}
\usepackage{tikz}

\newcommand{\ifff}{ if and only if }
\newcommand{\pc}{\mathbf{P}}

\newcommand{\pp}{\mathbb{P}}

\begin{document}
\setcounter{page}{1}

\title{Combinatorial Aspects of the Distribution of Rough Objects}
\author{\textbf{A. Mani}}
\institute{Department of Pure Mathematics\\
University of Calcutta\\
9/1B, Jatin Bagchi Road\\
Kolkata(Calcutta)-700029, India\\
Email: \email{a.mani.cms@gmail.com}\\
Homepage: \url{http://www.logicamani.in}}

\maketitle

\begin{abstract}
The inverse problem of general rough sets, considered by the present author in some of her earlier papers, in one of its manifestations is essentially the question of when an agent's view about crisp and non crisp objects over a set of objects has a rough evolution. In this research the nature of the problem is examined from number-theoretic and combinatorial perspectives under very few assumptions about the nature of data and some necessary conditions are proved.
\end{abstract}

\section{Introduction}

It is well known that sets of rough objects (in various senses) are quasi or partially orderable. Specifically in classical or Pawlak rough sets \cite{ZPB}, the set of roughly equivalent sets has a Quasi Boolean order on it while the set of rough and crisp objects is Boolean ordered.  In the classical semantic domain or classical meta level, associated with general rough sets, the collection of crisp and rough objects is quasi or partially orderable (or more generally a parthood relation is definable). Under minimal assumptions on the nature of these objects, many orders with rough ontology can be  associated - these necessarily have to do with concepts of discernibility. There are situations where the following can happen:
\begin{itemize}
\item {Case-1: The order structure is simple and objects are labeled crisp or otherwise - this can happen for example in communication between agents. }
\item {Case-2: Some of the order structure is lost, but objects are labeled crisp or otherwise - this can happen for example in communication between agents. }
\item {Case-3: The parthood relation among crisp objects is known fully. But the parthood relation on the set of crisp and non crisp objects is not known in full.}
\item {Case-4: All/Some of the order structure is lost, but objects are labeled crisp or otherwise, but the agent had applied some non-rough method of arriving at the information  - this again can happen for example in communication between agents.}
\item {Case-5: All objects are labeled and the parthood relation is known in full.}
\end{itemize}

Under minimal assumptions, it is shown that lot more information about possible models for the situation can be  
deduced. In this order structures relatable to chains and antichains are chosen because of their value in representation of finite posets and certain other possible posets have been excluded. 

The general problem that is investigated has the following form: \emph{given some information about number theoretic properties of crisp and non-crisp objects, then what can be said about the existence of models of possible rough ontology in the context?} 

So the problem falls under the general class of inverse problems (\cite{AM240}), but with additional assumptions of a number theoretic nature. New application contexts can be found in studies in social sciences, psychology and human reasoning - in these contexts finite bounds on possible number of non-crisp objects and their distribution has significant effect (especially the latter). This approach also has the potential to bridge the gap between theoretical rough sets and algorithms used in practice. Two papers that approach the bridging aspect without any number-theoretic assumptions are \cite{NOV,YY15}.

In algorithms used in practice, the goals can be very different - but a dominant method is to \emph{coerce data to follow preconceived models}. Here the idea of \emph{preconceived} is dictated by factors like \emph{ease of computation} at the cost of representation. This is reflected for example in the book \cite{MoPiZi}. Similar methods of intrusion can be found in other branches of machine learning - decision trees can be used for prediction and knowledge representation and in general the methods adopted are to intrusively optimize over the following principles:
\begin{itemize}
\item {Decision trees are comprehensible when they have lesser number of nodes in the tree.}
\item {Decision tree algorithms work faster when the depth of the tree is smaller.}
\item {Accuracy of decisions increase with decrease in number of misclassifications.}
\end{itemize}

In this paper, the motivations are towards knowledge representation. 

\subsection*{Background}

Let $S$ be any set and $l, u$ be lower and upper approximation operators on $\mathcal{S}\subseteq \wp (S)$ that satisfy monotonicity and $(\forall A\subseteq S)\, A\subseteq A^u$. An element $A\in\mathcal{S}$ will be said to be \emph{lower definite} (resp. \emph{upper definite}) if and only if $A^l = A$ (resp. $A^u = A$) and \emph{definite}, when it is both lower and upper definite. 

In general rough sets, these ideas of definiteness are insufficient as it can happen that upper approximations of upper definites are not upper definite. The following variants are of natural interest (though surprisingly their characterizations in the rough context are not fully understood): 
\begin{itemize}
\item {An upper definite object $A$ will be said to be \emph{strongly upper definite} if and only if \[A = A^u = A^{uu}. \] }
\item {An object $A$ will be said to be \emph{pre-strongly upper definite} if and only if \[A^u = A \,\& \, (\exists n) \, A^{u^{n}} = A^{u^{n+1}} .\] }
\item {An object $A$ will be said to be \emph{upper pre-definite} if and only if \[ (\exists n) \, A^{u^{n}} = A^{u^{n+1}} .\] }
\end{itemize}
Analogous concepts of lower definiteness and definiteness can be directly defined .

Possible concepts of rough objects considered in the literature include the following: 
\begin{itemize}
\item {A \textsf{non definite subset} of $S$, that is $A$ is a rough object \ifff $A^l \neq A^u$. }
\item {\textsf{Any pair of definite subsets} of the form $(A , B)$ satisfying $A\subseteq B.$}
\item {\textsf{Any pair of subsets} of the form $(A^l ,A^u)$.}
\item {Sets in an \textsf{interval of the form} $(A^l, A^u)$.}
\item {Sets in an \textsf{interval of the form} $(A, B)$ satisfying $A\subseteq B$ and $A, B$ being definite subsets.}
\item {A \textsf{non-definite element in a RYS}, that is an $x$ satisfying $\neg \pc x^u x^l   $}
\item {An \textsf{interval of the form}, $(A, B)$ satisfying $A\subseteq B$ and $A, B$ being definite subsets.}
\end{itemize}

The idea of definite and rough objects can be varied substantially even when the approximations have been fixed and the above concepts are based on representation. 

\textsf{Concepts of representation of objects necessarily relate to choice of semantic frameworks. In general, in most contexts, the order theoretic representations are of interest. In operator centric approaches, the problem is also about finding ideal representations. The central problem that is pursued in the present paper relates to combinatorial characterization and number theoretic properties of existence of models.}

Granular operator spaces, a set framework with operators introduced by the present author in \cite{AM6999}, will be used as all considerations will require quasi/partial orders in an essential way. The evolution of the operators need not be induced by a cover or a relation (corresponding to cover or relation based systems respectively), but these would be special cases. The generalization to some rough Y-systems \textsf{RYS} (see \cite{AM240} for definitions), will of course be possible as a result. 

\begin{definition}
A \emph{Granular Operator Space} $S$ will be a structure of the form $S\,=\, \left\langle \underline{S}, \mathcal{G}, l , u\right\rangle$ with $\underline{S}$ being a set, $\mathcal{G}$ an \emph{admissible granulation}(defined below) over $S$ and $l, u$ being operators $:\wp(\underline{S})\longmapsto \wp(\underline{S})$ satisfying the following:

\begin{align*}
A^l \subseteq A\,\&\,A^{ll} = A^l \,\&\, A^{u} \subset A^{uu}  \\
(A\subseteq B \longrightarrow A^l \subseteq B^l \,\&\,A^u \subseteq B^u)\\
\emptyset^l = \emptyset \,\&\,\emptyset^u = \emptyset \,\&\,S^{l}\subseteq S \,\&\, S^{u}\subseteq S.
\end{align*}

Here, \textsf{Admissible granulations} are granulations $\mathcal{G}$ that satisfy the following three conditions (Relative \textsf{RYS} \cite{AM240}, $\pc = \subseteq$, $\pp = \subset$) and $t$ is a term operation formed from set operations):

\begin{align*}
(\forall x \exists
y_{1},\ldots y_{r}\in \mathcal{G})\, t(y_{1},\,y_{2}, \ldots \,y_{r})=x^{l} \\
\tag{Weak RA, WRA} \mathrm{and}\: (\forall x)\,(\exists
y_{1},\,\ldots\,y_{r}\in \mathcal{G})\,t(y_{1},\,y_{2}, \ldots \,y_{r}) =
x^{u},\\
\tag{Lower Stability, LS}{(\forall y \in
\mathcal{G})(\forall {x\in \underline{S} })\, ( y\subseteq x\,\longrightarrow\, y \subseteq (x^{l})),}\\
\tag{Full Underlap, FU}{(\forall
x,\,y\in\mathcal{G})(\exists
z\in \underline{S} )\, x\subset z,\,y \subset z\,\&\,z^{l} = z^{u} = z,}
\end{align*}
\end{definition}

On $\wp(\underline{S})$, the relation $\sqsubset$ is defined by \[A \sqsubset B \text{ if and only if } A^l \subseteq B^l \,\&\, A^u \subseteq B^u.\] The rough equality relation on $\wp(\underline{S})$ is defined via $A\approx B \text{ if and only if } A\sqsubset B  \, \&\,B \sqsubset A$. 

Regarding the quotient $\underline{S}|\approx$ as a subset of $\wp(\underline{S})$, the order $\Subset$ will be defined as per \[\alpha \Subset \beta \text{ if and only if } \alpha^l \subseteq \beta^l \,\&\, \alpha^u \subseteq \beta^u.\] Here $\alpha^l$ is being interpreted as the lower approximation of any of the elements of $\alpha$ and so on. $\Subset$ will be referred to as the \emph{basic rough order}.

\begin{definition}
By a \emph{roughly consistent object} will be meant a set of subsets of $\underline{S}$ of the form  $H = \{A ; (\forall B\in H)\,A^l =B^l, A^u = B^u \}$. The set of all roughly consistent objects is partially ordered by the inclusion relation. Relative this maximal roughly consistent objects will be referred to as \emph{rough objects}. By \emph{definite rough objects}, will be meant rough objects of the form $H$ that satisfy 
\[(\forall A \in H) \, A^{ll} = A^l \,\&\, A^{uu} = A^{u}. \] 
\end{definition}

However, this definition of rough objects will not necessarily be followed in this paper.

\begin{proposition}
$\Subset$ is a bounded partial order on $\underline{S}|\approx$. 
\end{proposition}
\begin{proof}
Reflexivity is obvious.  If $\alpha \Subset \beta$ and $\beta \Subset \alpha$, then it follows that $\alpha^l = \beta^l$ and $\alpha^u = \beta^u$ and so antisymmetry holds. 

If $\alpha \Subset \beta$, $\beta \Subset \gamma$, then the transitivity of set inclusion induces transitivity of $\Subset$.
The poset is bounded by $0 = (\emptyset , \emptyset)$ and $1 = (S^l , S^u)$. Note that $1$ need not coincide with $(S, S)$. 
\qed
\end{proof}

In quasi or partially ordered sets, sets of mutually incomparable  elements are called \emph{antichains}. Some of the basic properties may be found in \cite{GG1998,koh}. The possibility of using antichains of rough objects for a possible semantics was mentioned in \cite{AM3690,AM9501} and has been developed subsequently in \cite{AM6999}. In the paper the developed semantics is applicable for a large class of operator based rough sets including specific cases of \textsf{RYS} \cite{AM240} and other less general approaches like \cite{CD3,CC5,YY9,IT2}. In \cite{CD3,CC5}, negation like operators are assumed in general and these are not definable operations relative the order related operations/relation.

\subsection{Concepts of Finite Posets}\label{wth}

Let $S$ be a finite poset with $\#(S) = n < \infty$. The following concepts and notations will be used in this paper:
\begin{itemize}
\item {If $\mathfrak{F}$ is a collection of subsets $\{X_i\}_{i\in J}$ of a set $X$, then a \emph{system of distinct representatives} \textsf{SDR} for $\mathfrak{F}$ is a set $\{x_i ; i\in J\}$ of distinct elements satisfying $(\forall i \in J) x_i \in X_i$.}
\item {For $a, b\in S$, $b$ covers $b$ will be denoted by $a \prec b$. $c(S)$ shall be the number of covering pairs in $S$.}
\item {A chain $C$ will be said to be \emph{saturated} if and only of if $a\prec_{|C} b$ (that is if $b$ covers a in the induced order $\leq_{|C}$ on $C$) implies $a\prec b$  } 
\item {A \emph{chain cover} of a finite Poset $S$ is a collection $\mathcal{C}$ of chains in $S$ satisfying $\cup \mathcal{C} = S$. It is disjoint if the chains in the cover are pairwise disjoint.}
\item {$S$ has finite width $w$ if and only if it can be partitioned into $w$ number of chains, but not less.}
\item {The \emph{Hasse index} of $S$ is defined by $i(S) = \dfrac{c(S)}{\# (S)}$.} 
\item {$S$ is \emph{graded} if there is a unique partition of $S$ into $\{A_i\}_{i=o}^{r}$ with $A_o$ being the set of minimal elements, and \[\forall x\in A_i (x\prec y \longrightarrow y\in A_{i+1}).\] $A_i$ are termed the \emph{levels} of $S$ and if $x\in A_i$, then \emph{rank} of $x$ is $rk(x) = i$.}
\item {A \emph{symmetric chain} $x_o \prec x_1 \prec \ldots x_r$ is a chain with $rk(x_i) = i$ for all $i$. A \emph{symmetric chain decomposition} of $S$ is a partition of $S$ into symmetric chains.}
\end{itemize}

The following theorems are well known:
\begin{theorem}
\begin{itemize}
\item {A collection of subsets $\mathfrak{F}$ of a finite set $S$ with $\#(\mathfrak{F}= r$ has an SDR if and only if for any $1 \leq k \leq r $,the union of any $k$ members of $\mathfrak{F}$ has size at least $k$, that is 
\[(\forall{X_1,\ldots , X_k \in \mathfrak{F}}) \, k \leq \#(\cup X_i).\] }
\item {Every finite Poset $S$ has a disjoint chain cover of width $w = width(S)$.}
\end{itemize}
\end{theorem}

\section{Semantic Framework}\label{sf}

Since \textsf{objects} are assumed to be definable if at all by sets of \textsf{attributes}, all considerations can be in terms of attributes.  For the considerations of the following sections on distribution of rough objects and on counting to be valid, a minimal set of assumptions are necessary. These are as follows:

\begin{align*}
\tag{FO1} S \text{ is a granular operator space}. \\
\tag{FO2} \mathbb{S} \subseteq \wp(\underline{S}) .\\
\tag{FO3} \# (\mathbb{S}) = n < \infty .\\
\tag{RO1} R \subset \mathbb{S} \text{- the set of rough objects in some sense}.\\
\tag{RO2} \# (R) = n-k < n .\\
\tag{CO1} C \subseteq \mathbb{S} \text{ is the set of crisp objects}.\\
\tag{CO2} \# (C) = k .\\
\tag{RC1} R\cap C = \emptyset .\\
\tag{RO3} \text{ there exists a map } \varphi : R \longmapsto C^{2} .\\
\tag{RC2} (\forall x\in R)(\exists a, b\in C) \varphi (x) = (a, b) \,\& \, a\subset b.
\end{align*}

Note that no further assumptions are made about the nature of $\varphi(x)$. It is not required that 
\[\varphi (x) = (a, b) \, \&\, x^l = a \, \&\, x^u = b,\]
though this happens often.

The set of crisp objects is necessarily partially ordered. In specific cases, this order may be a lattice, distributive, relatively complemented or Boolean order. Naturally the combinatorial features associated with granular operator space depends on the nature of the partial order and results in situations that is way more involved than the situation encoded by the following simple proposition.

\begin{proposition}
Given a fixed value of $\# (\mathbb{S}) = n = \# ((\wp (S))$ and $\#(C)=k$, $R$ must be representable by a finite subset $K \subseteq C^2 \setminus C$.   
\end{proposition}

The two most extreme cases of the ordering of the set $C$ of crisp objects correspond to $C$ forming a chain and $C\setminus \{0,1 \}$ forming an anti-chain. 
\begin{definition}
For $a, b \in R$, let 
\[\nu(a, b ) =\,\left\{
\begin{array}{ll}
0, & \text{ if } a^l = b^l   \, \& \, a^u = b^u  \\
1 & \text{ if } a^l \neq b^l \, \& \, a^u \neq b^u \\
\frac{1}{\pi} & \text{ if } a^l \neq b^l   \, \& \, a^u = b^u \\
\frac{1}{e} & \text{ if } a^l = b^l   \, \& \, a^u \neq b^u.
\end{array}
\right. \]

By the \emph{rough distribution index} of $R$ will be meant the sum
\[\iota(R,C) = \sum_{a, b\in R} \nu(a,b)\]
and the \emph{relative rough distribution index} of $R$ shall be
\[\iota^\ast (R, C) = \dfrac{\iota(R,C)}{(n-k)^2}\]
\end{definition}

\begin{theorem}
\[0 \leq \iota (R, C) \leq (n -k)^2\]
\end{theorem}

\begin{proof}
The lower bounds have been obtained on the assumption that the non crisp elements are mutually roughly equal.

\qed
\end{proof}

The measure gives an idea of the extent of distribution of non crisp objects over the distribution of the crisp objects and the relative measure is a bad approximation of the idea of seeking comparison across distributions of crisp objects.

\section{Example Contexts}

Often in the design, implementation and analysis of surveys (in the social sciences in particular), a number of intrusive assumptions on the sample are done and preconceived ideas about the population may influence survey design. Some assumptions that ensure that the sample is representative are obviously good, but as statistical methods are often abused \cite{MHRL} a minimal approach can help in preventing errors. 

The idea of samples being representative translates into number of non crisp objects being at least above a certain number and below a certain number. 

There are also situations (as when prior information is not available or ideas of representative samples are unclear) when such bounds may not be definable or of limited interest.

\section{Distribution of Objects in Chains: Case-0}

This a variant of the simplest case relative computations and requires three additional assumptions:
\begin{enumerate}
 \item {$C$ forms a chain under inclusion order.}
 \item {$\varphi$ is a bijection.}
 \item {Pairs of the form $(x, x)$, with $x$ being a crisp object, also correspond to rough objects.}
\end{enumerate}

It should be noted that this interpretation is not compatible with the interval way of representing rough objects without additional tweaking. It is the pairs interpretation that is being targeted.

\begin{theorem}
Under the above two assumptions, the number of crisp objects is related to the total number of objects by the formula: \[k \stackrel{i}{=} \dfrac{(1+4n)^\frac{1}{2} -1 }{2}.\] 
In the formula $\stackrel{i}{=}$ is to be read as \emph{if the right hand side (RHS) is an integer then the left hand side is the same as RHS.}  
\end{theorem}

\begin{proof}
\begin{itemize}
\item {Clearly the number of rough objects is $n - k$ .}
\item {By the nature of the surjection $n - k$ maps to $k^2$  pairs of crisp objects.}
\item {So $n - k = k^2$.}
\item {So integral values of $\dfrac{(1+4n)^\frac{1}{2} -1 }{2}$ will work.}  
\end{itemize}
\qed
\end{proof}

This result is associated with the distribution of odd square integers of the form $4n +1$ which in turn should necessarily be of the form $4(p^2 +p) +1$ (p being any integer). The requirement that these be perfect squares causes the distribution of crisp objects to be very sparse with increasing values of $n$.  The number of rough objects between two successive crisp objects increases in a linear way, but this is a misleading aspect. These are illustrated in the graphs below (\textsf{Fig.1,Fig.2}).

\begin{figure*}[htb]
\centering
\includegraphics[width=11.7cm]{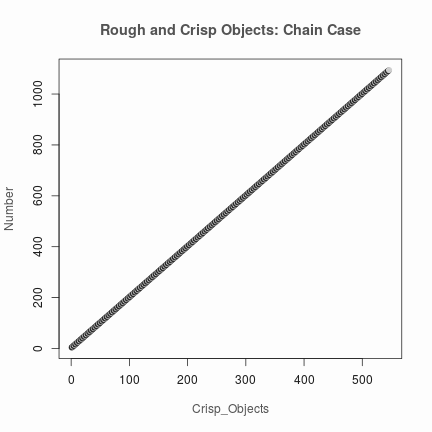}
\caption{Rough Objects Between Crisp Objects: Special Chain Case}
\end{figure*}

\begin{figure*}[htb]
\centering
\includegraphics[width=11.7cm]{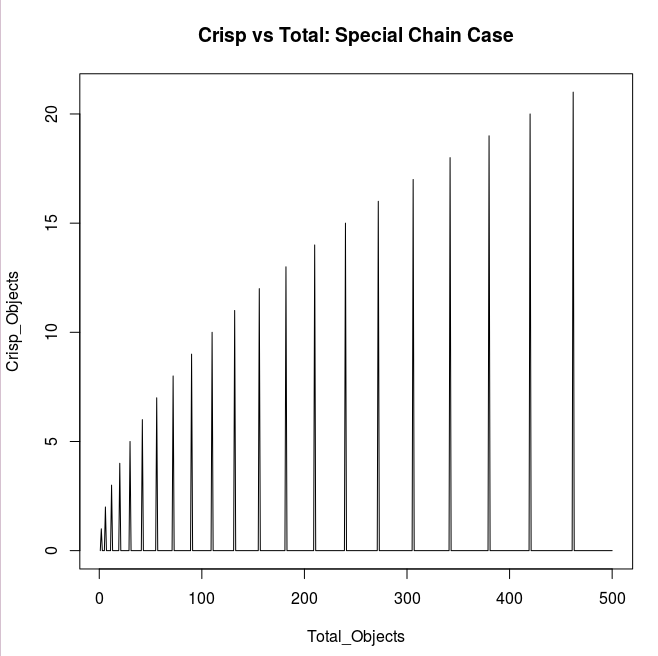}
\caption{Values of n and k: Special Chain Case}
\end{figure*}
\section{Distribution of Objects Over Chains: Case-1}

The simplest case relative computations requires two additional assumptions:
\begin{enumerate}
 \item {$C$ forms a chain under inclusion order.}
 \item {$\varphi$ is an bijection onto $C^2 \setminus \Delta_C$ ($\Delta_C$ being the diagonal of $C$).}
\end{enumerate}

\begin{theorem}
Under the above two assumptions, the number of crisp objects is related to the total number of objects by the formula: \[n - k = k^2 - k\] 
So, it is necessary that $n$ be a perfect square
\end{theorem}

\begin{proof}
\begin{itemize}
\item {Clearly the number of rough objects is $n - k$ .}
\item {By the nature of the surjection $n - k$ maps to $k^2 -k$  pairs of crisp objects (as the diagonal cannot represent rough objects).}
\item {So $n - k = k^2 - k$.}
\item {So $n= k^2$ is necessary.}  
\end{itemize}
\qed
\end{proof}

\begin{theorem}
In the above context, the cardinality of Boolean algebras that are power sets and in which the rough objects form chains in the induced order correspond to integral solutions for $x$ in  \[ 2^x = k^2 .\]  
\end{theorem}

\begin{proof}
As the number of elements in a finite power set must be of the form $2^x$ for some positive integer $x$, the correspondence follows. If $2^x = k^2$, then $x = {2 \log_2 k}$

This translates to a very sparse distribution of such  models. In fact for $n\leq 10^8$, the total number of models is $27$.

\qed
\end{proof}

\textsf{Fig. 3} gives an idea of the numbers that work:

\begin{figure*}[hbt]
\centering
\includegraphics[width=11.7cm]{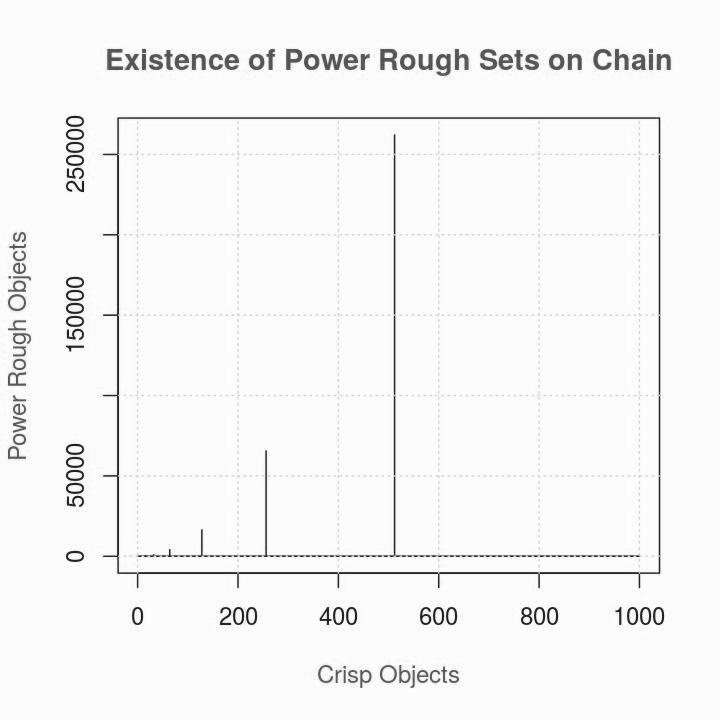}
\caption{Existence of Power Rough Sets on Chain}
\end{figure*}

\section{Distribution of Objects Over Chains: Case-2}

The formalism of this case is motivated by the hope that existence of rough objects may be strongly regulated by essentially number-theoretic properties.  

In the context of case-1, if the condition of $\varphi$ being a bijection into $C^2 \setminus \Delta_C$ is relaxed to $\varphi $ being an injection and it is assumed that \[\# (\varphi(R)) \leq \alpha (k^2 - k),\] for some rational $\alpha\in (0, 1]$ (the interpretation of $\alpha$ being that of a loose upper bound rather than an exact one), then the following theorems holds:

\begin{theorem}\label{fract}
Given fixed $n$, the possible values of $k$ correspond to integral solutions of the formula:
\[k = \dfrac{(\pi - 1) + \sqrt{(1-\pi)^2 +4n\pi}}{2\pi},\] subject to $k\leq \lfloor\sqrt{n}\rfloor$, $\# (\varphi(R)) = \pi (k^2 - k)$ and $0 < \pi \leq \alpha$.
\end{theorem}

\begin{proof}
 \begin{itemize}
 \item {When $n-k = \pi (k^2- k )$ then $\pi = \dfrac{(n - k)}{(k^2 - k)}$}
 \item {So positive integral solutions of $k = \dfrac{(\pi - 1) + \sqrt{(1-\pi)^2 +4n\pi}}{2\pi}$ may be admissible.}
 \item {The expression for $\alpha$ means that it can only take a finite set of values given $n$ as possible values of $k$ must be in the set $\{2, 3, \ldots , \lfloor\sqrt {\frac{n}{\alpha}}\rfloor\}$. The bounds for $k$ is not the best possible.}
 \end{itemize}
\qed
\end{proof}

\begin{theorem}
In the proof of the above theorem (Thm. \ref{fract}), fixed values of $n$ and $\pi$ do not in general correspond to unique values of $k$ and unique models. 
\end{theorem}

For fixed $n$ and possible values of $\pi$, the number of values of $k$ for which rough objects exist follows the 
pattern described in \textsf{Fig. 4} below. The intended reading is \emph{For $\pi=0.5$ and $n=1000000$, the number of values of $k$ that work seems to be $1413$}.

\begin{figure*}[hbt]
\centering
\includegraphics[width=11.7cm]{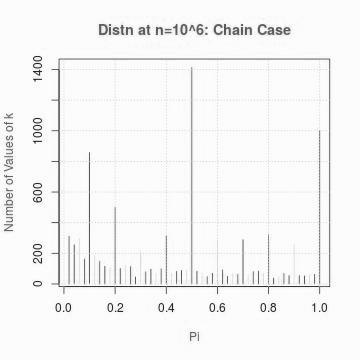}
\caption{Number of Possible Values of $k$}
\end{figure*}

If the bounds on $k$ are imposed on the graph in \textsf{Fig. 4} then \textsf{Fig. 5} is the result:

\begin{figure*}[hbt]
\centering
\includegraphics[width=11.7cm]{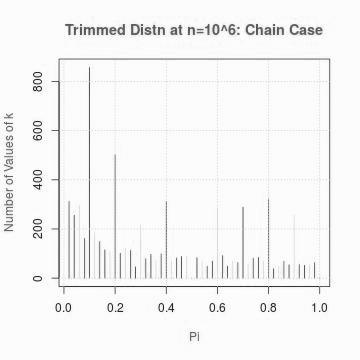}
\caption{Trimmed Number of Possible Values of $k$}
\end{figure*}

\subsection*{Algorithms: Case-2}

An algorithm for computing admissible values of $alpha$ can be
 
 \begin{enumerate}
\item {Fix the value of $n$.} 
\item {Start from possible values of $k$ less than $\sqrt{n - 1}$.}
\item {Compute $\alpha$ for all of these values.}
\item {Suppose the computed values are $\alpha_1, \ldots \alpha_r$}
\item {Check the admissibility of solutions.}
\end{enumerate}

Another algorithm for converging to solutions is the following:
\begin{enumerate}
\item {Start from a sequence $\{\alpha_i\}$ of possible values in the interval $(0,1)$.}
\item {Check the admissibility and closeness to solutions}
\item {If a solution appears to be between $\alpha_i$ and $\alpha_{i+1}$, add an equally spaced subsequence between the two.}
\item {Check the admissibility and closeness to solutions.}
\item {Continue}
\item {Stop when solution is found}
\end{enumerate}

\begin{theorem}
Both of the above algorithms converge in a finite number of steps. 
\end{theorem}

\begin{proof}
Convergence of the first algorithm is obvious.

Convergence of the second follows from the following construction:
\begin{itemize}
\item {Suppose the goal is to converge to an $\alpha \in (0,1) $.}
\item {Let $\alpha_o =0, \, \alpha_1 = 1 $ and for a fixed positive integer $n$ and $i= 1, \ldots, n$, let $\alpha_{1i} = \frac{i}{n}$ and $\alpha \in (\alpha_{1j}, \alpha_{1 j+1})$. }
\item {Form $n$ number of equally spaced partitions $\{\alpha_{2i}\}$ of $(\alpha_{1j}, \alpha_{1 j+1})$ and let $\alpha \in (\alpha_{2j}, \alpha_{2 j+1})$.}
\item {Clearly $(\forall \epsilon >0 \, \exists N\, \forall r>N)\, |\alpha - \alpha_{rj}| < \epsilon $}
\item {So the algorithm will succeed in finding the required $\alpha$. }
\end{itemize}
\qed 
\end{proof}

\section{Bounded Distribution on Chains}

The idea of bounded distribution corresponds to the set $R$ being partitioned into disjoint subsets of size $\{r_i \}_{i=1}^{g}$ with $g= k^2 -k$ subject to the condition $\beta$\[a \leq r_i \leq b \leq n -k, \text{  with  }a,\, b \text{ being constants.} \]

\begin{theorem}
If the crisp objects form a chain, then the total number of possible models $B$ is \[B \,= \sum_{\alpha\in \pi(r)|\beta}\prod_{i=1}^{k^2 -k} \alpha_i \text{  and  }n_o a^{k^2 - k} \leq B \leq n_o b^{k^2 - k} ,\]
with the summation being over partitions $\alpha = \{\alpha_i \}$ of $r$ subject to the condition $\beta$ and $n_o$ being the number of admissible partitions under the conditions.
\end{theorem}

\begin{proof}
On a chain of length $k$, $k^2 -k$ spaces can be filled.

The next step is to determine the partitions $\pi(r)$ of $r$ into $k^2 -k$ distinct parts.

The condition $\beta$ eliminates many of these partitions resulting in the admissible set of partitions $\pi(r)|\beta$. 

Each of the partitions $\alpha \in \pi(r)|\beta$ corresponds to $\prod_i \alpha_i $ number of possibilities.

So the result follows.

\qed

\end{proof}

\section{Distribution of Objects: Most General Set-Theoretic Context}

In general, in the context of the framework specified in the Sec.\ref{sf} it can be assumed that 
\begin{itemize}
\item {$\# (\varphi(R)) = t \leq n-k$,}
\item {$t = \beta (k^2 -k)$ and,}
\item {$n - k = \alpha (k^2 - k)$. for some constants $t,\, \beta,\, \alpha$}
\end{itemize}

This can also be used when objects are neither crisp or non crisp. In practice, objects may be neither crisp nor clearly non crisp possibly when:
\begin{itemize}
\item {a consistent method of identifying crisp objects is not used or}
\item {some objects are merely labeled on the basis of poorly defined partials of features or}
\item {a sufficiently rich set of features that can provide for consistent identification.}
\end{itemize}

The following construction provides a way of integrating the order structure on the set of all crisp and non crisp objects into basic computational considerations:

\begin{definition}
The \emph{lower definable scope} $\mathbf{SL}(x)$ of an element $x\in R$ will be the set of maximal elements in $\downarrow (x) \cap C$, that is \[\mathbf{SL}(x) = \max(\downarrow(x) \cap C). \]

The \emph{upper definable scope} $\mathbf{SU}(x)$ of an element $x\in R$ will be the set of minimal elements in $\uparrow (x) \cap C$, that is \[\mathbf{SU}(x) = \min(\uparrow(x) \cap C). \]
\end{definition}

All representations of rough objects can be seen as the result of choice operations \[\psi_x : \mathbf{SL}(x) \times \mathbf{SU}(x) \longmapsto C^2 \setminus \Delta_C . \] Letting $\# (\mathbf{SL} (x)) = c(x)$ and $\# (\mathbf{SU} (x)) = v(x)$ formulas for possible values are obtainable. Finding a simplification without additional assumptions remains an open problem though.

\subsection*{Chain Covers}

Let $C^*$ be the set of crisp objects $C$ with the induced partial order, then by the theorem in Sec \ref{wth},  
The order structure of the poset of crisp objects $C^*$ permits a disjoint chain cover. This permits a strategy for estimating the structure of possible models and counting the number of models. 

\begin{itemize}
\item {Let $\{C_i \,:\, i=1, \ldots h \}$ be a disjoint chain cover of $C^*$. Chains starting from $a$ and ending at $b$ will be denoted by $[[a, b]]$.}
\item {Let $C_1$ be the chain $[[0, 1]]$ from the the smallest(empty) to the largest object.}
\item {If $C_1$ has no branching points, then without loss of generality, it can be assumed that $C_2 = [[c_{2l}, c_{2g}]]$ is another chain with least element $c_{2l}$ and greatest element $c_{2g}$ such that $0 \prec c_{21}$, possibly $c_{2g} \prec 1$ and certainly $c_{2g} < 1$. }
\item {If $c_{2g} < 1$, then the least element of at least two other chains ($[[c_{3l}, \, c_{3g}]]$ and $[[c_{4l}, \, c_{4g}]]$) must cover $c_{2g}$, that is $c_{2g}\prec c_{3l}$ and $c_{2g}\prec c_{4l}$.}
\item {This process can be extended till the whole poset is covered. }
\item {The first step for distributing the rough objects amongst these crisp objects consists in identifying the spaces distributed over maximal chains on the disjoint cover subject to avoiding over counting of parts of chains below branching points.}
\end{itemize}

The above motivates the following combinatorial problem for solving the general problem: 

Let $H = [[c_l,c_g]] $ be a chain of crisp objects with $\#(H) = \alpha$ and let $c_o$ be a branching point on the chain with $\# ([[c_l , c_o]]) = \alpha_o$. Let \[S_C = \{(a, b) \,; \,a, b\in [[c_o, c_g]] \text{ or } c_l < a, b < c_o \}.\] In how many ways can a subset $R_f \subseteq R$ of rough objects be distributed over $S_C$ under $\# (R_f) = \pi$?

\begin{theorem}
If the number of possible ways of distributing $\pi$ rough objects over a chain of $\alpha$ crisp elements is $n(\pi,\alpha )$, then the number of models in the above problem is \[n (\pi,\alpha ) - n (\pi , \alpha_o).\]  
\end{theorem}

\begin{proof}
This is because the places between crisp objects in $[[c_l,c_o ]] $ must be omitted. The exact expression of $n(\pi, \alpha)$ has already been described earlier.
\qed 
\end{proof}

Using the above theorem it is possible to evaluate the models starting with splitting of $r$ into atmost $w$ partitions. \emph{Because of this it is not necessary to use principal order filters generated by crisp objects to arrive at direct counts of the number of possible cases and a representation schematics}. 

\section{Interpretation and Directions}

The results proved in this research are relevant from multiple perspectives. In the perspective that does not bother with issues of contamination, the results mean that the number of rough models relative the number of other possible models of computational intelligence is low. This can be disputed as the signature of the model is restricted and categoricity does not hold. 

In the perspective of the contamination problem, the axiomatic approach to granules , the results help in handling inverse problems in particular. From a minimum of information, it can also be deduced
\begin{itemize}
\item {whether a rough model is possible or}
\item {whether a rough model is not possible or}
\item {whether the given data is part of some minimal rough extensions}
\end{itemize}
The last possibility can be solved by keeping fixed the number of rough objects or otherwise. These problems apply for the contaminated approach too. It should be noted that extensions need to make sense in the first place.
The results are also expected to have many applications in probabilist approaches and variants.

All this is despite the paper being among the simplest in the literature on rough sets.

\bibliographystyle{splncs.bst}
\bibliography{biblioam09022016+.bib}
\end{document}